\relax
\documentclass[letterpaper]{article}
\usepackage[final]{nips_2016} 
\usepackage{times}
\usepackage{helvet}
\usepackage{courier}
\usepackage{amssymb}
\usepackage{graphicx}
\usepackage{algorithm}
\usepackage{algorithmic}
\newcommand{\learner}{\mathcal{L}}
\newcommand{\target}{c^*}
\newcommand{\scriptC}{\mathcal{C}}

\newcommand{\scriptF}{\mathcal{F}}
\newcommand{\real}{\mathbb{R}}
{\vspace{0.1in}}
\newtheorem{proposition}{\vspace{0.1in}{\noindent \bf Proposition}}{\vspace{0.1in}}
\newtheorem{lemma}{\vspace{0.1in}{\noindent \bf Lemma}}{\vspace{0.1in}}
\newenvironment{proof}{\vspace{0.1in}{\noindent \bf Proof}}{\vspace{0.1in}}

\begin{document}
%
\title{A Characterization of Prediction Errors}
\author{Christopher Meek \\Microsoft Research\\One Microsoft Way\\Redmond, WA 98052}
\maketitle
\begin{abstract}
Understanding prediction errors and determining how to fix them is critical to building effective predictive systems. In this paper, we delineate four types of prediction errors (mislabeling, representation, learner and boundary errors) and demonstrate that these four types characterize all prediction errors. In addition, we describe potential remedies and tools that can be used to reduce the uncertainty when trying to determine the source of a prediction error and when trying to take action to remove a prediction error.
\end{abstract}

\section{Introduction}

Prediction errors arise in interactive machine learning systems (e.g., Fails and Olsen 2003), \nocite{Fails:2003:IML:604045.604056} machine teaching (e.g. Simard et al 2014), \nocite{ICEarXiv} and when statisticians, scientists and engineers build predictive systems. Our goal in this paper is to provide an exhaustive categorization of the types of prediction errors and to provide guidance on actions one can take to remedy prediction errors. We suspect that this will be helpful to both expert and non-expert users trying to leverage machine learning and statistical models in building predictive systems.

Our characterization of prediction errors has four top-level categories; mislabeling, representation, learner, and boundary errors. Each of these error types are associated with specific deficiencies that, when identified, are potentially remedied.  Furthermore, we prove that the categorization into these error types is sufficient to characterize all prediction errors. 

We also suggest actions that can be taken to detect and remove prediction errors. 
With the aim of removing an entire type of prediction error from consideration we introduce the concept of a consistent learning algorithm. We demonstrate that there are consistent learning algorithms and describe how, when consistent learning algorithms are used, none of the prediction errors are learner errors. We also describe how a teacher might benefit from the identification of an {\em invalidation set}; a minimal set of labeled examples that contain one or more prediction errors. Finally we consider the implications of these results for developing teaching protocols that help the teacher to take appropriate actions to remedy prediction errors.

\section{Related Work}

The problem of debugging statistical models has been studies in a number of contexts. An excellent example of this work is the work by Amershi et al (2015) who also provides references to other related work. \nocite{Amershi:2015:MRP:2702123.2702509} 
Our categorization of prediction errors extends the informal categorization provided by Amershi et al (2015).
In that work, the authors describe potential sources of prediction errors in developing tools for identifying and exploring prediction errors. Specifically they consider three sources of errors; insufficient data, feature deficiencies, and mislabeled data. In our categorization, errors of insufficient data are a specific type of learner error that we call an objective error (they do not consider other types of learner errors), feature deficiencies are a specific type of representation error that we call feature blindness and mislabeled data is what we call mislabeling errors. Amershi et al (2015) do not consider boundary errors. 

The concept of an invalidation set is related to a number of existing concepts in the theory of machine learning include the exclusion dimension (Angluin 1994), \nocite{Angluin:2004:QR:982360.982362} the unique specification dimension (Hedigus 1995), \nocite{Hegedus:1995:GTD:225298.225311} and the certificate size (Hellerstein et al 1996). \nocite{Hellerstein:1996:MQN:234752.234755} Our focus, however, is on teaching with both labels and features whereas previous work considers only teaching with labels. 

\section{Prediction Errors}

In this section, we define the set of prediction errors that can arise when a teacher teaches a machine to classify objects by providing labeled examples and features. In addition, we provide essential definitions for the remainder of the paper.

We are interested in building a classifier of objects. We use $x$ and $x_i$ to denote particular objects and $X$ to denote the set of objects of interest. We use $y$ and $y_i$ for particular labels and $Y$ to denote the space of possible labels. For binary classification $Y=\{0,1\}$.  A classification function is a function from $X$ to $Y$.\footnote{Note that, while we call this mapping a classification function, the definition encompasses a broad class of prediction problems including structured prediction, entity extraction, and regression.} The set of classification functions is denoted by $\scriptC = X\rightarrow Y$. We use $\target$ to denote the target classification function that the teacher wants to teach the machine to implement.

One essential ingredient that a teacher provides are features or functions which map objects to scalar values. A {\em feature} $f_i$ (or $g_i$) is a function from objects to real numbers (i.e. $f_i \in X\rightarrow \real$). 
We denote the set of teachable feature functions by $R=\{f_1, f_2,\ldots \}$ and call a finite subset of $R$ a {\em feature set} (i.e., $F\subset 2^R$).
Clearly not all feature functions are directly teachable --- if the target classification function were teachable then we would not need to provide labeled examples.
The feature set $F_i=\{f_{i,1},\ldots,f_{i,p}\}$ is $p$-dimensional. We use a $p$-dimensional feature set to map an object to a point in $\real^p$. We denote the mapped object $x_k$ using feature set $F_i$ by $F_i (x_k )=(f_{i,1} (x_k ),\ldots,f_{i,p} (x_k ))$ where the result is a vector of length $p$ where the $j^{th}$ entry is the result of applying the $j^{th}$ feature function in $F_i$ to the object. 

Another essential ingredient that a teacher provides is a training set, a set of labeled examples.
A {\em training set} $T\subset X\times Y$ is a set of labeled examples.
We say that the training set $T$ has $n$ examples if $|T|=n$ and denote the set of training examples as $\{(x_1,y_1),\ldots,(x_n,y_n)\}$. A training set is unfeaturized. We use feature sets to create featurized training sets.  For $p$-dimensional feature set $F_i$ and an $n$ example training set $T$ we denote the featurized training set $F_i(T)= \{(F_i (x_1 ),y_1 ),\ldots,(F_i (x_n ),y_n)\}\in \{\real^p\times Y\}^n$. We call the resulting training set an $F_i$ featurized training set or the $F_i$ featurization of training set $T$. 

The method by which the machine learns a classification function is called a learning algorithm. A learning algorithm is, in fact, a set of learning algorithms as we now describe. First, a $d$-dimensional learning algorithm $\ell_d$ is a function that takes a $p$-dimensional feature set $F$ and a training set $T$ and outputs a function $h_p\in \real^p\rightarrow Y$. Thus, the output $h_p$ of a learning algorithm using $F_i$ and training set $T$ can be composed with the functions in the feature set to yield a classification function of objects (i.e., $h_p\circ F_i\in \scriptC$). The hypothesis space of a $d$-dimensional learning algorithm $\ell_d$ is the image of the function $\ell_d$ and is denoted by $H_{\ell_d}$ (or $H_d$ if there is no risk of confusion). 
A classification function $c\in \scriptC$ is {\em consistent with a training set} $T$ if $\forall (x,y)\in T$ it is the case that $c(x)=y$. A $d$-dimensional learning algorithm $\ell_d$ is {\em consistent} if the learning algorithm outputs a hypothesis consistent with the training set whenever there is a hypothesis in $H_d$ that is consistent with the training set. A {\em vector learning algorithm} $\ell=\{\ell_0,\ell_1,\ldots \}$ is a set of $d$-dimensional learning algorithms one for each dimensionality.  A {\em consistent} vector learning algorithm is one in which each of the $d$-dimensional learning algorithms is consistent. Finally, a {\em (feature-vector) learning algorithm} $\learner$ takes a feature set $F$, a training set $T$, and a vector learning algorithm $\ell$ and returns a classification function in $\scriptC$. In particular $\learner_\ell(F,T)=\ell_{|F|}(F,T)\circ F\in \scriptC$. 
We say that a classification function $c$ is $F$-$\learner$-learnable if there exists a training set $T$ such that $\learner(F,T)=c$. We denote the set of $F$-$\learner$-learnable functions by $\scriptC(F,\learner)$.
When the vector learning algorithm is clear from context or we are discussing a generic vector learning algorithm we drop the $\ell$ and write $\learner(F,T)$. One important property of a feature set is whether it is sufficient to teach the target classification function $\target$.
A feature feature set $F$ is {\em sufficient} for learner $\learner$ and target classification function $\target$ if $\target$ is $F$-$\learner$ learnable (i.e. $\target\in \scriptC(F,\learner)$).

The central component of an interactive machine learning system for teaching a classification function is a teaching protocol. A teaching protocol is the method by which a teacher teaches a machine learning algorithm. While not our primary focus in this paper, our interest in teaching protocols, is that they (1)  provide a means of illustrating the potential value of the results we provide and (2) provide a valuable avenue for future exploration as alternative teaching protocols provide different types of support for teachers in their efforts to build a classifier.  

Finally, we define a definition for a prediction error.
An object $x\in X$ is a {\em prediction error} for training set $T$, feature set $F$, and learning algorithm $\learner$ if the trained classifier $\learner(F,T)=c$ does not agree with the target classification function on the object (i.e., $c(x)\neq \target(x)$).
We distinguish two types of prediction errors; a {\em training set prediction error} in which the prediction errors is on an object in the training set $x\in T_X= \{x|(x,y)\in T\}$ and a {\em generalization error} in  which the object is not in the training set (i.e., $x \in X\setminus T_X$).

\section{A Characterization of Prediction Errors}

In this section, we develop a categorization for prediction errors considering both training set and generalization errors. We also demonstrate that our categorization is exhaustive, that is, we provide a characterization of prediction errors. Our categorization is relative to a particular training set $T$, feature set $F$, and learning algorithm $\learner$. We describe four categories of errors: mislabeling errors, representation errors, learner errors, and boundary errors. Generalization errors are of a different nature than training set prediction errors due to the fact that they are not in the training set. This difference is important because the teacher can only see a generalization error when they provide a label for an object not in the training set. We classify the types of generalization errors  relative to a particular training set $T$, feature set $F$, and learning algorithm $\learner$ by considering the result of adding a correctly labeled version of the object to the training set (i.e., for generalization error $x\in X\setminus T_X$ we use training set $T'=T\cup \{(x,\target(x))\})$.

\subsection{Mislabeling Errors}

A {\em mislabeling error} is a labeled object such that the label does not agree with the target classification function (i.e., a labeled example $(x,y)$ such that $y\neq \target(x)$). 
At first glance it is not clear that mislabeling errors have anything to do with a prediction error, however,  mislabeling errors can give rise to prediction errors. In particular, if the learned classifier matches the label of a mislabeled object there will be a prediction error. For instance, if we have only one labeled object $(x,1)$ in a training set and it is mislabeled then any consistent classifier will result in a prediction error. This type of prediction error arises due to an error by the teacher (a.k.a. labeler) who provided an incorrectly labeled example. We assume that a teacher, when confronted with a mislabeling error can correct the label to match the target classification function. In practice this may not be the case due a number of factors, including lack of clarity about the target classification function $\target$ and teacher error (see, e.g., Kulesza et al 2014). \nocite{Kulesza:2014:SLF:2556288.2557238}

\subsection{Learner Errors}

A {\em learner error} is a prediction error that arises due to the fact that the learner does not find a classification function that correctly predict the training set when such a learnable classifier exists (i.e., $\exists c\in \scriptC(F,\learner) \forall (x,y)\in T \ c(x)=\target(x)$ and $\forall c \in \scriptC(F,\learner)$ if $(\forall (x,y)\in T \ c(x)=\target(x))$ then $\learner(F,T)\neq c$). Note that when considering a generalization error we use the augmented training set $T'$. Typical learning algorithms select a function from the possible learnable classification function $\scriptC(F,\learner)$ using a fitness function or loss function. In this case, it is natural to consider two types of learner errors; {\em optimization errors} and {\em objective errors}. In an optimization error, there is a learnable classification function that correctly classifies the training set and the consistent classification function has a lower loss than $\learner(F,T)$. In an objective error, all learnable classification functions that correctly classify the training set have higher loss than $\learner(F,T)$.

\subsection{Representation Errors}

A {\em representation error} is a prediction error that arises due to the fact that there is no learnable classification function that correctly predicts the training set (i.e., $\forall c\in \scriptC(F,\learner) \exists (x,y)\in T s.t.\  c(x)\neq \target(x)$). Again, for a generalization error we use the augmented training set $T'$. Representation errors arise due to a limitation of the feature set, a limitation of the learning algorithm or both. More specifically, an error can arise due to the feature-blindness of the learning algorithm --- it does not have access to features that distinguish objects --- or that the hypothesis class of the learning algorithm is impoverished (e.g., trying to learn the x-or function with a linear classifiers).

\subsection{Boundary Errors}

Our final type of prediction error is a type of generalization error. A {\em boundary error} is a prediction error for an object $x$ if adding $(x,\target(x))$ to the training set yields a classification function $c'$ that correctly predicts the augmented training set (i.e., $c=\learner(F,T)$ and $c'=\learner(F,T')$ and $c(x)\neq c'(x)=\target(x)$). 

\subsection{Characterization of Prediction Errors}

We conclude this section by providing characterizations of training set prediction errors and of prediction errors. Our first proposition demonstrates that there are three types of training set prediction errors.

\begin{proposition} \label{thm:trainerr}
If there is a training set prediction error given a training set, feature set, and learning algorithm then there is either a mislabeling, representation, or learner error.
\end{proposition}
 
\begin{proof}
Let $x$ be a training set prediction error for training set $T$, feature set $F$ and learning algorithm $\learner$. That means that there exists $(x,y)\in T$ such that $c=\learner(F,T)$ and $c(x)\neq \target(x)$. If there are mislabeled examples in $T$ we are done. If there are no mislabeled examples then it must be the case that either there is or is not a classification function in $\scriptC(F,\learner)$ that correctly classifies $T$. If there is such a classification function then we have a learner error and if not we have a representation error.
\end{proof}

The following Proposition demonstrates that the only other type of error required to capture the types of prediction errors is the boundary error.

\begin{proposition} \label{thm:prederr}
If there is a prediction error given a training set, feature set, and learning algorithm then there is either a mislabeling, representation, learner, or boundary error.
\end{proposition}
 
\begin{proof}
Let $x$ be a prediction error for training set $T$, feature set $F$ and learning algorithm $\learner$. We consider two cases:

Case 1: $x$ is a training set prediction error. This case is handled in Proposition~\ref{thm:trainerr}.

Case 2: $x$ is a generalization error and there is no training set prediction error for $F$, $T$, and $\learner$. 
In this case, we consider the augmented training set  $T'=T\cup (x,\target(x))$ to identify the type of prediction error for $x$.
If $\learner(F,T')$ is consistent with the training set $T'$ we have a boundary error. Otherwise, as described in case 1, there must either be a learner error or representation error. Note that while it might be the case that there are mislabeled objects not included in the training set, such mislabeling errors are not generalization errors and not relevant as the mislabeling cannot be the source of a prediction error because it is not included in the training set. Thus, every generalization error can be associated with one three prediction error types.
\end{proof}

\section{Detecting and Removing Types of Prediction Errors}

In this section we discuss the problem of identifying the type of a prediction error that arises while a teacher teaches a classification function. We also discuss a potential approach to reducing the effort required by the teacher to identify and remove prediction errors.

\subsection{Detecting Boundary Errors}

A boundary error is a generalization error, an error for the currently trained classification function to correctly classify an unseen object. A boundary error can only arise in a teaching protocol in which there are labeled examples that are not included in a training set. The most common scenario where this happens is when there is a test set that is used to obtain an estimate of the prediction performance of the learned classification function. A teaching protocol can automatically detect whether prediction error is a boundary error by including the example in the training set and determining if the resulting classification function correctly predicts the error. A teaching protocol can potentially leverage such a test to choose when to sample examples, for instance, sampling more examples in a region with demonstrable ignorance about the boundary. This is related to the motivation for using uncertainty sampling as an active learning strategy (Settles 2012). \nocite{settles.book12}

\subsection{Detecting and Removing Learner Errors}

It is possible to completely eliminate learner errors by choosing to use consistent learning algorithm as the following proposition demonstrates.

\begin{proposition} \label{thm:consistentlearner}
If there is a training set prediction error for feature set $F$ and consistent learning algorithm $\learner$ then the error must be a mislabeling or representation error.
\end{proposition}

\begin{proof}
Recall the definition of a consistent learning algorithm; 
a consistent learning algorithm returns a classification function that correctly predicts the training set if there is a learnable classification function that does so. To prove the proposition we assume that there is a prediction error that is not a mislabeling or a representation error. From Proposition~\ref{thm:trainerr} we know that there must be a learner error. In this case, there is a learnable classification function that correctly classifies the training set. From the consistency of $\learner$ and the lack of representation or mislabeling errors, we know that $\learner(F,T)$ correctly classifies $T$ which implies there is no training set prediction error which is a contradiction.
\end{proof}

We have demonstrated that consistent learning functions can be used to eliminate learner errors. Next we demonstrate that consistent learning algorithms exist. 

\begin{proposition}\label{thm:logreg}
Maximum-likelihood logistic regression is a consistent learner.
\end{proposition}

We have moved the proofs for several  propositions to the end of the paper to improve readability.

The fact that maximum likelihood logistic regression is a consistent learner is due in part to the fact that the optimization problem is convex. It is also due to fact that we have restricted the functional form of the classification function to be a generalized linear form limiting the capacity of the learning algorithm. The following example demonstrates that we need not limit the capacity of the learning algorithm to have a consistent learning algorithm.

\begin{proposition} \label{thm:1NN}
One nearest-neighbor (1NN) is a consistent learner.
\end{proposition}

Recall that there are two types of learner errors; optimization and objective errors. Next we illustrate how objective errors can arise when applying learning algorithms to prediction problems. The most common way in which objective errors arise is when one adds {\em regularization} to reduce generalization error. For logistic regression, this adds a penalty to the loss function that penalizes the length of the weight vector (i.e., $-\lambda ||w||$). Figure~\ref{fig:LogRegInconsistent} illustrates the 0.5 decision boundaries for different choices of regularization parameter $\lambda$. With $\lambda=0$ we obtain a consistent learning algorithm but with $\lambda=0.5$ and $\lambda=1.0$ we see examples of objective errors. Similarly, if we consider $k$-nearest-neighbor algorithms for $k>1$ there are training sets that can fail to correctly classify the training set due to the fact that the prediction for an object $x$ in the training set is also a function of $k-1$ other points that might disagree on the prediction at $x$.

\begin{figure}
\begin{center}
\includegraphics[width=1.9in]{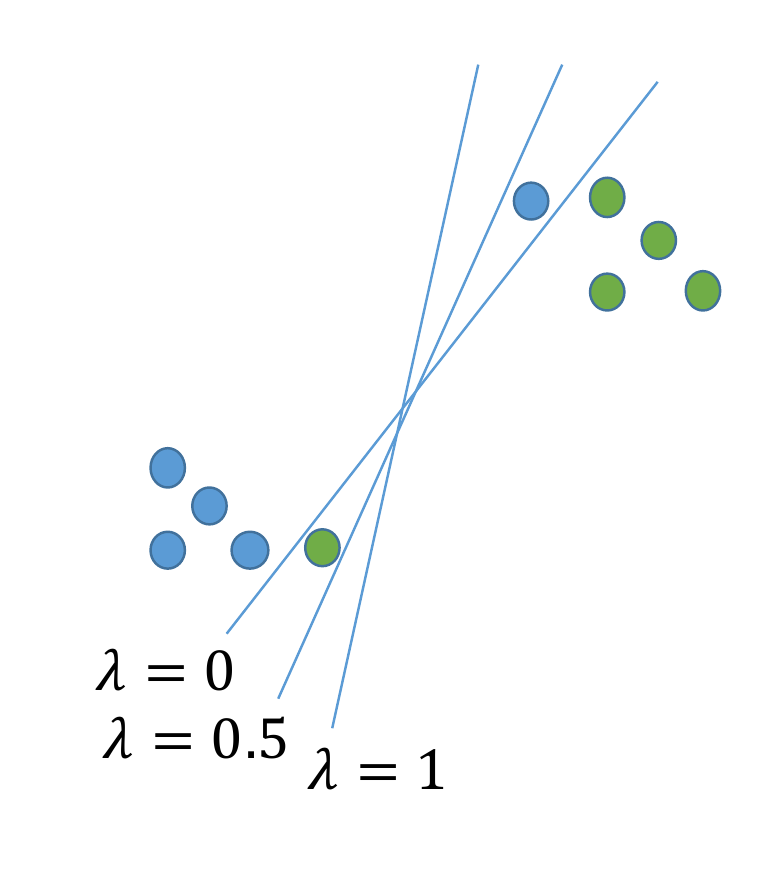}
\end{center}
\caption{\label{fig:LogRegInconsistent} A two-dimensional example demonstrating that regularized logistic regression can be inconsistent.}
\end{figure}

As argued above, we can remove learning errors from consideration by using a consistent learner. It is not clear whether this is the best approach in all teaching scenarios. It might be the case that it is beneficial to the teaching process to use an inconsistent learner to, for instance, improve generalization performance. In such circumstances, one might be able to leverage a family of learning algorithms in which there is a regularization parameter that can control the potential for objective errors. An example of such a family is the family of $\lambda$ regularized logistic regression learners. When using such a family, one can detect a learner error by training with different settings of the regularization parameter.

\subsection{Representation and Mislabeling Errors}

Next we consider the problem of detecting representation and mislabeling errors assuming that we have no learner errors. It follows, for instance, from Proposition~\ref{thm:consistentlearner} that this is the situation when using a consistent learning algorithm. 

In general, we cannot distinguish between mislabeling errors and representation errors. To see this, consider a binary classification training set of two objects $\{(x_1, 1),(x_2,0)\}$. In this situation, it is possible that the target classification function is the constant function $\target(x)=1$ and the label for $x_2$ is a mislabeling error or that there is a feature function $f_1$ that distinguishes the two objects (e.g., $f_1(x_1)=5$ and $f_1(x_2)=7$) in which case there is a representation error.

While one cannot hope to automatically distinguish mislabeling and representation errors, one can hope that the teacher can detect and distinguish such errors when they are presented to the teacher. One way in which a teaching protocol might help the teacher to detect and diagnose representation and mislabeling errors is by identifying a small set of labeled examples to inspect. We propose the use of 
an invalidation set for this purpose. An {\em invalidation set} is a training set of minimal size containing a prediction error. By identifying a minimal training set with a prediction error we aim to reduce the effort required by the teacher to determine whether prediction errors are mislabeling errors or representation errors.

The next results bounds the size of an invalidation set for any consistent linear learner including maximum-likelihood logistic regression and the one nearest neighbor classifier.

\begin{proposition}\label{thm:LinInvalidBound}
If $T$ has a prediction error for target concept $\target$ using feature set $F$ and $\learner$ where $\learner$ is a consistent linear learner then an invalidation set has at most $|F|+2$ examples.
\end{proposition}

\begin{proposition}\label{thm:1NNInvalidBound}
If $T$ has a prediction error for target concept $\target$ using feature set $F$ and $\learner$ where $\learner$ is a one-nearest-neighbor learner then an invalidation set has at most $2$ examples.
\end{proposition}

\subsection{Discussion}

We begin our discussion by presenting a teaching protocol through which a teacher might teach a machine a classification function.  This provides a means to both summarize our results and highlight open issues.

\begin{algorithm}[H]
\caption{Error-Driven-Teaching-Protocol}
\label{alg:error}
\begin{algorithmic}
\STATE{\bf Input} consistent learning algorithm $\learner$, set of objects $X$.
\STATE $T=\{\}$ \hspace{1in} // training set $T\subset X\times Y$
\STATE $F=\{\}$ \hspace{1in} // feature set $F\in \scriptF$
\STATE $c=\learner(T,F)$;
\WHILE {!\underline{Terminate}()}
\STATE $(x,y) =$ \underline{ Add-labeled-example}($X,F,T,\learner$);
\STATE $T=T\cup (x,y)$;
\STATE $c=\learner(T,F)$; //remove boundary errors by retraining
\WHILE{($\exists (x,y)\in T$ such that $c(x)\neq y$)}
\STATE Identify invalidation set $T'\subset T$
\STATE found-mislabeled-example =\underline{Check-labels}($T'$)
\IF {(found-mislabeled-example)}
\STATE \underline{Correct-Labels}($T'$)  //fix mislabeling error
\ELSE 
\STATE \underline{Add-feature}();  //fix representation error
\ENDIF
\ENDWHILE
\ENDWHILE
\STATE return c;
\end{algorithmic}
\end{algorithm}

Algorithm~\ref{alg:error} describes a teaching protocol that illustrates one potential use of our categorization of prediction errors. The teaching protocol uses the teacher to address particular sub-problems as indicated by the underlined function calls. In particular, the teacher is required to determine whether to terminate the teaching session, to choose a new example to label for the training set, to check and correct labels and to add features.

The teaching protocol in Algorithm~\ref{alg:error} assumes the use of a consistent learning algorithm which removes the need to consider learner errors. After adding a new labeled example, the classifier is immediately retrained to remove any potential boundary errors. Finally, we use the concept of an invalidation set to reduce the effort required to identify and correct mislabeling and representation errors.

This is an idealized teaching protocol but, as such, points to important research directions for providing support for teachers. These directions include support for choosing which item to select and label, choosing which feature to add, choosing when to terminate the teaching effort and support exploration of the space of objects and evolution of the target classification function.

\section{Proofs}
In this section we provide the proof for several proposition. Some of the proofs rely on convex geometry and linear algebra. We assume that the reader is familiar with basic concepts and elementary results from convex geometry and linear algebra. We denote the convex closure of a set of points by $conv(S)$.

\noindent {\bf Proposition \ref{thm:logreg}} {\em
Maximum-likelihood logistic regression is a consistent learner.}

\begin{proof}
We consider binary classification $Y=\{0,1\}$ using a $d$-dimensional feature set $F$. We use $w\in \real^d, b\in \real$ to parameterize our logistic regression. The likelihood function for logistic regression is $Pr(Y=y|X=x,F,w,b)=exp((w\cdot F(x) + b) y)/{(1+exp(w\cdot F(x) + b))}$. The maximum-likelihood estimator is $ArgMax_{w,b} \prod_{(x_i,y_i)\in T} Pr(y_i | F(x_i), w, b)$. This function is a convex function and, as such, we can guarantee that we do not have optimization errors. The likelihood maps featurized objects to real numbers and is thus not a binary classification function. We can map a likelihood function into a classification function via a threshold. We will use a threshold of 0.5 and thus $c(x)=1$ if $Pr(Y=1|X=x,F,w,b)>0.5$ and $c(x)=0$ otherwise.

We reparameterize the likelihood function using the following definitions: $w'=w/{||w||}$, $b'=-b/{||w||}$,$\beta=||w||$, and $d(x,w',b',F)=w'\cdot F(x) -b'$ where $||\cdot||$ is the Euclidean length of the vector $w$. The likelihood function is then $Pr(Y=y|X=x,F, w',b',\beta)=exp(\beta d(x,w',b',F)y)/{(1+ exp(\beta d(x, w', b',F)))}$. This parameterization has a natural interpretation. The decision boundary (probability 0.5) for logistic regression is $Hyp_{w,b,F}=\{x|w\cdot F(x) +b=0\}$ and the function $d(x,w',b',F)$ is the signed distance of a point to the decision boundary. The parameter $\beta$ controls the steepness of the logistic function (e.g., the slope of the likelihood at a point on the decision boundary in the direction normal to the decision boundary). It follows that if a set of points is linearly separable then the limiting likelihood is 1. In particular, using any separating hyperplane to define $w'$ and increasing the slope parameter $\beta$ will increase the likelihood with the likelihood approaching 1.

To prove the claim we assume that maximum-likelihood logistic regression is not consistent. In this case, there exists a feature set $F$ and a training set $T$ such that there is a learnable classification function $c$ using $F$ and maximum-likelihood logistic regression such that $c$ correctly classifies $T$ but that the classification function $c'=\learner(F,T)$ does not correctly classify $T$. Note that due to the convexity of the optimization problem we do not have an optimization error which implies that the prediction error is an objective error. To prove the claim we need to demonstrate that cannot be the case. At this point we know that there must be a labeled example $(x,y)\in T$ such that $c'(x) \neq y$. In this case, the point $F(x)$ is on the incorrect side of the decision boundary and thus the likelihood for that point is at most $1/2$. The likelihood on the other points is at most 1. Thus the maximum likelihood obtainable on a training set with at least one error is at most $1/2$. We argued above, however, that the likelihood on a separable problem will approach 1 thus we have a contradiction.

\end{proof}

\noindent {\bf Proposition \ref{thm:1NN}} {\em
One nearest-neighbor (1NN) is a consistent learner.}

\begin{proof}
Nearest-neighbor algorithms are memorization learning algorithms. As defined above, a training set can only have one label per object (i.e., $T\subset X\times Y$). It is straight-forward to relax this assumption but we choose not to do so here. A given $d$-dimensional feature set $F$ might map multiple training set objects to the same point in $\real^d$ in which case there is not a unique nearest neighbor. In this case, we assume that the 1NN algorithm chooses one a canonical object from the set of zero distance neighbors (e.g., according to some ordering over the objects). In this case, if all of the objects in each of these zero distance neighbor sets (subsets of the training set) has the same target label then the resulting classifier is consistent. If, however, there is a set of zero distance neighbors that contain objects with different target labels then the resulting classifier is not consistent but, in this case, no consistent 1NN classification function using $F$ is possible.
\end{proof}

\vspace{-0.1in}
\begin{lemma}{[Kirchberger 1903; Shimrat 1955]}\label{lem:invalid} \nocite{Kirchberger1903, shimrat1955}
Two finite sets $S,T\subset \real^d$ are strictly separable by some hyperplane if and only if for every set $U$ consisting of at most $d+2$ points from $S\cup T$ the sets $U\cap S$ and $U\cap T$ can be strictly separated.
\end{lemma}

\noindent {\bf Proposition \ref{thm:LinInvalidBound}} {\em
If $T$ has a prediction error for target concept $\target$ using feature set $F$ and $\learner$ where $\learner$ is a consistent linear learner then an invalidation set has at most $|F|+2$ examples.}

\begin{proof}
Let $X$ be our set of objects. Again define Define $S=\{F(x)\in \real^d| x\in X$ and $\target(x)=1\}$ and $T=\{F(x)\in \real^d | x\in X$ and $\target(x)=0\}$. 
Because $F$ is not linearly sufficient there is no separating hyperplane for $S$ and $T$. From Lemma~\ref{lem:invalid} and the fact that $F$ is $d$-dimensional, we know that there must be a subset $U \subset  \{F(x)|x\in X\}$ where $|U|\leq d+2$ and such that $U\cap S$ and $U\cap T$ are not separated by any hyperplane.
\end{proof}

\noindent {\bf Proposition \ref{thm:1NNInvalidBound}} {\em
If $T$ has a prediction error for target concept $\target$ using feature set $F$ and $\learner$ where $\learner$ is a one-nearest-neighbor learner then an invalidation set has at most $2$ examples.
}

\begin{proof}
A 1NN classifier can have an invalidation set for a $p$-dimensional feature set only if the feature set maps two objects with different class labels to the same point in $\real^p$. A set with one such object from each class is an invalidation set.
\end{proof}

\section{Acknowledgments}
Thanks to Patrice Simard, Max Chickering, Jina Suh, Carlos Garcia Jurado Suarez, and Xanda Schofield for helpful discussions about prediction errors.
\newpage

\bibliography{TeachingEffort}
\bibliographystyle{aaai}
\end{document}